\documentclass[11pt]{article}

\usepackage[preprint]{acl}

\usepackage{times}
\usepackage{latexsym}

\usepackage[T1]{fontenc}
\usepackage[utf8]{inputenc}
\usepackage{microtype}
\usepackage{inconsolata}

\definecolor{codegreen}{rgb}{0,0.6,0}
\definecolor{codegray}{rgb}{0.5,0.5,0.5}
\definecolor{codepurple}{rgb}{0.58,0,0.82}
\definecolor{backcolour}{rgb}{0.95,0.95,0.92}
\colorlet{reddish}{red!70}

\usepackage{enumitem}
\usepackage{booktabs}
\usepackage{multirow}
\usepackage{graphicx}
\usepackage{float}
\usepackage{amsmath,amssymb}
\usepackage{amsthm}
\usepackage{hyperref}
\usepackage{tcolorbox}
\usepackage{pifont}
\newcommand{\cmark}{\ding{51}}
\newcommand{\xmark}{\ding{55}}

\newtheorem{theorem}{Theorem}[section]
\newtheorem{lemma}[theorem]{Lemma}
\newtheorem{corollary}[theorem]{Corollary}

\newtheorem{assumption}{Assumption}

\newtheorem{remark}{Remark}
\DeclareMathOperator*{\argmax}{arg\,max}

\title{More Test-Time Compute Can Hurt: \\ Overestimation Bias in LLM Beam Search}

\author{
  \textbf{Gal Dalal\textsuperscript{1}\thanks{Equal contribution.}},
  \textbf{Assaf Hallak\textsuperscript{1}\footnotemark[1]},
  \textbf{Gal Chechik\textsuperscript{1,2}},
  \textbf{Yftah Ziser\textsuperscript{1,3}}
\\\\
  \textsuperscript{1}NVIDIA Research,
  \textsuperscript{2}Bar-Ilan University,
  \textsuperscript{3}University of Groningen
}

\begin{document}
\maketitle
\begin{abstract}
Wider beam search should improve LLM reasoning, but when should you stop widening? Prior work on beam width selection has focused on inference efficiency \citep{qin2025dsbd, freitag2017beam}, without analyzing whether wider search can \emph{hurt} output quality. We present an analysis, grounded in Extreme Value Theory, that answers this question. Beam selection over noisy scorer outputs introduces a systematic overestimation bias that grows with the candidate pool size, and we derive a maximum useful beam width $\hat{k}$ beyond which search degrades performance. This critical width depends on the signal-to-noise ratio of the scorer: $\hat{k}$ grows exponentially with $(\Delta/\sigma)^2$, where $\Delta > 0$ is the quality advantage of correct paths over incorrect ones and $\sigma$ is the scorer noise. We validate this theory by comparing perplexity-guided and PRM-guided beam search across three 7B-parameter models and ten domains on MR-BEN (5,975 questions). Perplexity scoring, with its high noise, yields $\hat{k} = 1$: search provides no benefit at any width tested. PRM scoring, with lower noise, yields $\hat{k} \geq 4$, with gains of up to 8.9 percentage points. The same model, the same algorithm, but different scorers place $\hat{k}$ at opposite ends of the beam width range. Our analysis identifies the scorer's signal-to-noise ratio as the key quantity governing beam width selection, and we propose diagnostic indicators for choosing the beam width in practice.
\end{abstract}

\begin{figure*}[t]
    \centering
    \includegraphics[width=\linewidth]{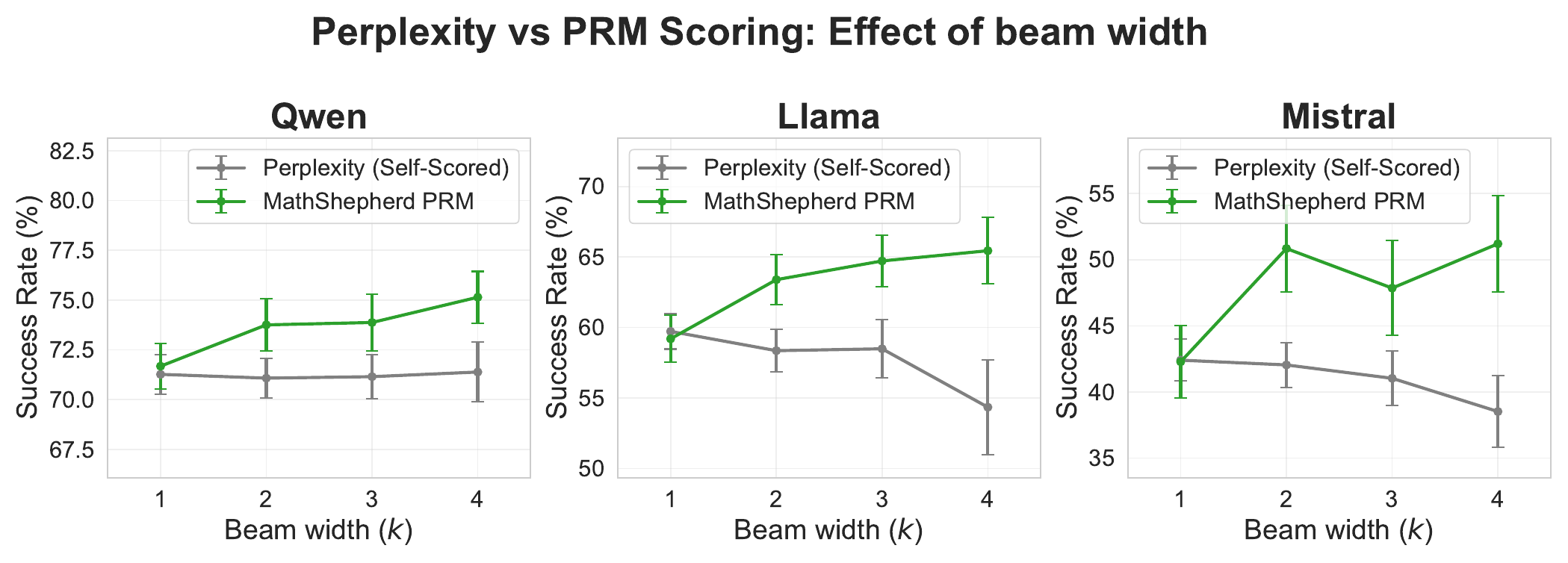}
    \caption{Per-model comparison of perplexity (gray) vs.\ PRM (green) scoring across beam widths. The two scorers place the predicted maximum useful beam width $\hat{k}$ at opposite ends of the range: perplexity curves are flat ($\hat{k} = 1$), while PRM curves rise through beam width 4 ($\hat{k} \geq 4$), illustrating how scorer quality determines the optimal beam width.}
    \label{fig:scorer_comparison}
\end{figure*}

\section{Introduction}

Large Language Models (LLMs) have demonstrated remarkable progress on complex reasoning tasks, in large part due to techniques that elicit and leverage step-by-step solutions \citep{wei2022cot}. A major recent thread is \emph{process supervision}: instead of judging only the final answer, a model (or learned verifier) provides feedback on intermediate steps, enabling search and credit assignment over reasoning traces. In structured domains such as mathematics, Process Reward Models (PRMs) that score each step have proven particularly effective \citep{lightman2023prm, cui2025prime}.

At the same time, there is growing interest in \emph{test-time compute scaling}---allocating additional inference-time computation to improve accuracy via sampling and search \citep{snell2024scaling}. Common instantiations include best-of-$n$ sampling, reranking, and beam search over partial chains of thought. A pervasive implicit assumption is that \emph{any} reasonable scoring signal can drive such search: generate more candidates, score them, and pick the best. This assumption has not been rigorously examined.

A natural candidate for a cheap scoring signal is \emph{perplexity}, a model's measure of predictive confidence. Several intrinsic signals have been proposed for self-evaluation, from entropy-based self-certainty \citep{kang2024selfcertainty} to Markovian reward adjustments \citep{ma2025metacognition}. We focus on perplexity as the most widely used and simplest such signal. It requires no step-level training and is available for any autoregressive model, making it a representative test of whether training-free intrinsic signals can drive reasoning-time search. If perplexity sufficed, one could avoid the expense of training specialized PRMs. Yet it is unclear whether perplexity is informative enough at the granularity of intermediate reasoning steps, or whether its noise causes search to select fluent but unfaithful steps.

A related overestimation phenomenon has been observed in model-based reinforcement learning, where tree search with learned value functions can degrade policy performance \citep{dalal2021improve}. Our analysis addresses the distinct setting of LLM beam search, where the scorer is a perplexity model or PRM and the search operates over natural language reasoning traces. We show that beam search introduces a systematic overestimation bias: an extreme-value effect in which the expected maximum score among incorrect candidates grows with search width. From this we derive an explicit criterion for when search is beneficial and a threshold beyond which widening degrades performance.

Our study addresses three questions, each answered by a corresponding contribution:
\begin{enumerate}
    \item \textbf{What determines the maximum useful beam width?} We develop a formal analysis based on Extreme Value Theory showing that beam selection introduces a systematic overestimation bias proportional to scorer noise $\sigma$. We derive a predicted maximum useful beam width $\hat{k}$ that depends on the scorer's signal-to-noise ratio $\Delta/\sigma$. Beyond $\hat{k}$, the bias exceeds the quality gap and search degrades performance, a phenomenon we call \emph{reward inversion}.
    \item \textbf{Does scorer quality change the answer?} We validate the theory across three model families and ten reasoning domains (Figure~\ref{fig:scorer_comparison}). Perplexity scoring yields no benefit at any beam width, while PRM scoring yields meaningful gains, as the theory predicts.
    \item \textbf{How should the beam width be chosen in practice?} We propose a principled approach to beam width selection based on diagnostic indicators such as score margins and pilot-run comparisons that signal whether a scorer can support wider search. While prior work adjusts beam width for inference speed \citep{qin2025dsbd}, our criterion targets output quality. We also show that investing in scorer quality is far more effective than widening the beam.
\end{enumerate}

\section{Background and Related Work}

\subsection{Beam Width Selection}
Prior work on beam width selection has focused on inference efficiency. \citet{freitag2017beam} introduced adaptive pruning strategies for neural machine translation, achieving up to 43\% speedup without degrading quality. \citet{qin2025dsbd} proposed Dynamic-Width Speculative Beam Decoding, dynamically adjusting beam count to balance latency and quality in LLMs.

These approaches treat beam width as an \emph{efficiency} parameter. We address a complementary question: when does wider search \emph{degrade} output quality? We show the answer depends on the scorer's signal-to-noise ratio and derive an explicit threshold beyond which additional candidates hurt performance.

\subsection{Process Supervision for LLM Reasoning}
Chain-of-Thought prompting showed that step-by-step reasoning improves LLM performance \citep{wei2022cot, chen2022pot}. Process Reward Models (PRMs) extend this by scoring intermediate steps \citep{lightman2023prm, zheng2025prmchallenges}, with Math-Shepherd \citep{wang2024mathshepherd} reducing annotation costs via Monte Carlo estimation. \citet{cobbe2021verifiers} showed that best-of-$N$ selection with a trained verifier matches a 30$\times$ parameter increase.

PRMs face reward hacking and distributional shift \citep{zheng2025prmchallenges, cui2025prime}. \citet{song2025overcredit} show that PRM false positives impose an asymptotic ceiling on best-of-$N$ performance, and \citet{luo2025soc} establish that implicit Q-functions learned during fine-tuning produce over-optimistic estimates that degrade beam search. Both findings are consistent with the overestimation mechanism we formalize in Section~\ref{sec:theory}.

\subsection{Inference-Time Search}
``Test-time compute scaling'' allocates additional inference computation to improve performance \citep{snell2024scaling}. \citet{wu2025inference} show that smaller models with advanced search achieve Pareto-optimal cost-performance trade-offs, and \citet{liu2025can1b} demonstrate that optimal strategies depend on the policy model, PRM, and problem difficulty. The success of models like DeepSeek-R1 \citep{deepseek2025r1} further motivates understanding these limits. These studies establish that more search \emph{can} help, but none characterize \emph{when it stops helping}---the question our analysis addresses.

Tree of Thoughts \citep{yao2023tree} frames reasoning as search over a tree of partial solutions. Beam search and best-of-$N$ sampling are common strategies in this paradigm \citep{xie2024beamsearch, kang2024selfcertainty}, relying on a scoring function such as a PRM or a proxy like perplexity \citep{jelinek1977perplexity}.

\subsection{Overestimation Bias and Reward Overoptimization}
Systematic overestimation from maximization over noisy estimates was first identified by \citet{thrun1993overestimation} in Q-learning, where function approximation combined with the $\max$ operator produces upward bias. \citet{dalal2021improve} extended this to tree search with learned value functions, showing that deeper search can degrade policy performance in model-based RL.

In the LLM setting, \citet{gao2023overoptimization} demonstrated reward overoptimization: optimizing against an imperfect reward model via best-of-$N$ sampling degrades true performance past a critical $N$---a manifestation of Goodhart's law. Their experiments represent a depth-1, width-$N$ special case of beam search. The critical $N$ at which performance degrades is analogous to $\hat{n}$ in our formulation, and the degradation is consistent with the overestimation mechanism that Lemma~\ref{lem:overestimation} quantifies via Extreme Value Theory.

\section{Theoretical Framework: Overestimation Bias in Beam Selection}
\label{sec:theory}

We develop a formal analysis of why beam search with noisy scorers can degrade performance. The core insight is structural: single-sample decoding returns one draw from the scorer and incurs no selection bias, while beam search selects from a pool of $n$ noisy draws, introducing a systematic upward bias in estimated quality that grows with both $n$ and the scorer noise $\sigma$. Our analysis identifies a predicted maximum useful candidate pool size $\hat{n}$: increasing the pool up to $\hat{n}$ is guaranteed to help, but beyond $\hat{n}$ performance may degrade. The predicted beam width $\hat{k}$ then follows from $\hat{n}$ via the expansion scheme and is not itself the universal quantity.

\subsection{Setup and Notation}

We analyze beam search at the \emph{reasoning-step} level: each ``token'' in our search tree is a complete reasoning step (a paragraph of chain-of-thought), not a single word. This is the granularity used by Tree of Thoughts \citep{yao2023tree} and PRM-guided search \citep{lightman2023prm}, unlike classical token-level beam search.

Consider a single beam selection step. The current beam of $k$ paths each generates $k$ candidate continuations, producing $n = k^2$ total candidates. This symmetric construction is a design choice, not a theoretical requirement: tying the expansion factor to the beam width keeps $k$ as the single tuning parameter and matches standard beam search implementations. The theory itself operates on the candidate pool size $n$ and applies regardless of how $n$ decomposes into beams and expansions.

Each candidate is an independent sample from the language model's conditional distribution given its parent path, scored in a separate forward pass. A scorer assigns each candidate $i$ a score $R_i$ modeled as:
\begin{equation}
    \label{eq:score_model}
    R_i = \mu_i + \epsilon_i,
\end{equation}
where $\mu_i \in \mathbb{R}$ is the \emph{true quality} of path $i$ and $\epsilon_i$ is scorer noise. Because candidates are sampled and scored independently, the noise terms $\epsilon_1, \dots, \epsilon_n$ are mutually independent by construction. Beam search selects the candidate with the highest score: $i^* = \argmax_{i} R_i$.

We use the term ``single-sample decoding'' rather than ``greedy'' because all experiments use temperature 0.7; beam width 1 simply means no selection step is performed. The structural distinction is between \emph{unselected} decoding: $k=1$, where the score is a single noisy draw with no selection bias; and \emph{beam-selected} decoding: $k > 1$, where the maximization over $k$ noisy scores introduces upward bias.

\subsection{Assumptions}

\begin{assumption}[Gaussian Homoscedastic Noise]
    \label{ass:noisy_scorer}
    The scorer noise terms are Gaussian with common variance: $\epsilon_i \sim \mathcal{N}(0, \sigma^2)$ for all $i = 1, \dots, n$.
\end{assumption}

The noise level $\sigma$ is a property of the scorer, not of individual candidates. A high-noise scorer (large $\sigma$, e.g.\ perplexity) assigns scores that are only loosely correlated with true quality; a low-noise scorer (small $\sigma$, e.g.\ a trained PRM) tracks quality more faithfully. The overestimation bias arises not from any asymmetry in noise across candidates, but from the structural asymmetry between \emph{unselected} decoding (one draw, no bias) and \emph{beam-selected} decoding (maximum of $k$ draws, upward bias).

\begin{assumption}[Two-Class Quality]
    \label{ass:two_class}
    Among the $n$ candidates, there is one \emph{correct-type} candidate with true quality $\mu_c$ and $n-1$ \emph{incorrect-type} candidates with common true quality $\mu_w < \mu_c$. Define the quality gap $\Delta := \mu_c - \mu_w > 0$.
\end{assumption}

This simplification isolates the effect of the maximization and represents the \emph{hardest case} for beam search: only one candidate is correct, so beam selection must identify it among $n-1$ equally plausible alternatives. If multiple candidates are correct (as often occurs in practice), at least one is more likely to survive beam selection, making the analysis strictly more favorable. The one-correct-candidate model therefore yields worst-case guarantees (see Remark~\ref{rem:stylized} for the precise sense in which the resulting $\hat{k}$ is a lower bound on the true safe beam width $k^*$).

\subsection{Overestimation Bias from Beam Selection}

The beam search selects the candidate with the highest score. When multiple noisy estimates are compared, the maximum is biased upward, a well-known consequence of order statistics. The following lemma quantifies this for our setting using the Generalized Extreme Value (GEV) distribution.

\begin{lemma}[Overestimation Bias]
    \label{lem:overestimation}
    Under Assumptions~\ref{ass:noisy_scorer} and \ref{ass:two_class}, let $n$ denote the total number of candidates with $n - 1$ incorrect-type candidates. The expected score of the best incorrect candidate is well-approximated by the GEV limit:
    \begin{equation}
        \label{eq:overestimation}
        \mathbb{E}\left[\max_{j=1,\dots,n-1} R_j\right] \approx \mu_w + B(\sigma, n-1),
    \end{equation}
    where the overestimation bias $B(\sigma, n-1)$ for $n \geq 3$ is given by
    \begin{equation}
        \label{eq:bias_term}
\sigma\!\left[
\Phi^{-1}\!\left(1-\tfrac{1}{n-1}\right)
+ \gamma_{\mathrm{EM}}
\frac{\sigma^{\mathrm{GEV}}(n-1)}{\sigma}
\right]
    \end{equation}
    with $\Phi^{-1}$ the standard normal quantile function, $\gamma_{\mathrm{EM}} \approx 0.5772$ the Euler--Mascheroni constant, and $\sigma^{\mathrm{GEV}}(n-1) = \sigma[\Phi^{-1}(1 - \frac{1}{\mathrm{e}(n-1)}) - \Phi^{-1}(1 - \frac{1}{n-1})]$ the GEV scale parameter, where $\mathrm{e}$ is Euler's number. For $n \gg 1$, this simplifies to:
    \begin{equation}
        \label{eq:bias_approx}
        B(\sigma, n-1) \approx \sigma\sqrt{2\log(n-1)}.
    \end{equation}
    The correct candidate's expected score is simply $\mathbb{E}[R_c] = \mu_c$ (no bias, since it is a single draw).
\end{lemma}

\paragraph{Proof outline.} The $n-1$ incorrect scores are i.i.d.\ Gaussian draws. By the Fisher--Tippett--Gnedenko theorem, their maximum converges to a Gumbel distribution whose location and scale parameters yield the bias term~\eqref{eq:bias_term}. The large-$n$ simplification~\eqref{eq:bias_approx} follows from a standard asymptotic expansion of the normal quantile. We defer the full proof to Appendix~\ref{app:proof_overestimation}.

The key asymmetry is structural, not distributional: the incorrect candidates receive a ``free bonus'' of approximately $\sigma\sqrt{2\log(n-1)}$ from the maximization, while the single correct candidate does not. Even though all candidates have the same noise $\sigma$, beam search has a systematic tendency to prefer incorrect paths when this bonus exceeds the quality gap $\Delta$.

\begin{remark}[Conservatism of the bound]
    \label{rem:stylized}
    Two modeling choices make our failure bound pessimistic, and both work in the same direction. First, we bound whether \emph{any} incorrect candidate outscores the correct one ($\Pr(\max_j R_j > R_c)$), whereas beam search retains the top $k$ candidates and eliminates the correct path only when at least $k$ incorrect candidates outscore it---a strictly harder condition to satisfy. Second, Assumption~\ref{ass:two_class} posits a single correct candidate; multiple correct paths only increase the chance that at least one survives. Because both choices \emph{overestimate} the failure probability, the predicted $\hat{k}$ is a \textbf{lower bound} on the true safe beam width $k^*$: beam search is guaranteed to help for $k \leq \hat{k}$, and may continue to help beyond.
\end{remark}

\subsection{Sub-Optimal Selection Probability}

We now bound the probability that beam search selects a sub-optimal path.

\begin{theorem}[Sub-Optimal Selection Bound]
    \label{thm:suboptimal}
    Under Assumptions~\ref{ass:noisy_scorer} and \ref{ass:two_class}, the probability that beam search selects an incorrect candidate over the correct one is bounded by:
    \begin{equation}
        \label{eq:suboptimal_bound}
        \Pr\!\left(\max_{j=1,\dots,n-1} R_j > R_c\right) \leq \left(1 + \frac{\Delta_{\mathrm{eff}}^2}{2\sigma^2}\right)^{-1},
    \end{equation}
    where $\Delta_{\mathrm{eff}} = \Delta - B(\sigma, n-1)$ is the effective quality gap after accounting for overestimation bias, and we require $\Delta_{\mathrm{eff}} > 0$.
\end{theorem}

\begin{proof}
    See Appendix~\ref{app:proof_suboptimal}.
\end{proof}

The bound is governed by the effective signal-to-noise ratio $\Delta_{\mathrm{eff}}/\sigma$: the true quality gap $\Delta$, reduced by the overestimation bias $B(\sigma, n-1)$, relative to the scorer noise $\sigma$. When scorer noise is high (perplexity), $B(\sigma, n-1)$ is large, shrinking $\Delta_{\mathrm{eff}}$ and making sub-optimal selection likely. When scorer noise is low (PRM), $B(\sigma, n-1)$ is small, preserving the true quality gap. By Remark~\ref{rem:stylized}, this bound is pessimistic: the resulting $\hat{k}$ is a lower bound on the true safe beam width $k^*$.

\subsection{When Does Search Help? A Necessary Condition}

The following corollary provides a practical criterion.

\begin{corollary}[Search Benefit Criterion]
    \label{cor:criterion}
    Beam search over $n$ candidates can improve over single-sample decoding only if the scorer's noise is small enough that the quality gap survives the overestimation bias, yielding the approximate necessary condition:
    \begin{equation}
        \label{eq:criterion}
        \Delta \gtrsim \sigma \sqrt{2\log(n-1)}.
    \end{equation}
    When this condition is violated, the overestimation bias exceeds the quality gap. Beam search then degrades performance by selecting incorrect paths whose scores have been inflated by the maximization.
\end{corollary}

\begin{proof}
    From Lemma~\ref{lem:overestimation}, beam search benefits require $\Delta_{\mathrm{eff}} = \Delta - B(\sigma, n-1) > 0$. Using the approximation $B(\sigma, n-1) \approx \sigma\sqrt{2\log(n-1)}$ yields the condition.
\end{proof}

Inverting the criterion \eqref{eq:criterion} yields the central object of this paper: the predicted maximum useful candidate pool size.

\begin{corollary}[Maximum Useful Candidate Pool]
    \label{cor:kstar}
    For a scorer with noise $\sigma$ and a problem with quality gap $\Delta > 0$, the predicted maximum number of candidates at which search is expected to help is:
    \begin{equation}
        \label{eq:nstar}
        \hat{n} = 1 + \exp\!\left(\frac{\Delta^2}{2\sigma^2}\right).
    \end{equation}
    For $n \leq \hat{n}$, the overestimation bias is smaller than the quality gap and search can improve over single-sample decoding. For $n > \hat{n}$, the bias dominates and search is expected to degrade performance.
\end{corollary}

\begin{proof}
    The condition $\Delta > \sigma\sqrt{2\log(n-1)}$ from Corollary~\ref{cor:criterion} is equivalent to $n - 1 < \exp(\Delta^2 / 2\sigma^2)$.
\end{proof}

The fundamental limit is the candidate pool size $\hat{n}$. The predicted beam width $\hat{k}$ depends on how candidates are generated. In our experiments, beam width $k$ produces $n = k^2$ candidates per selection step (\S\ref{sec:experiments}), giving $\hat{k} = \lfloor\sqrt{\hat{n}}\rfloor$. Other expansion strategies would yield a different mapping from $\hat{n}$ to $\hat{k}$. The dependence of $\hat{n}$ on the squared signal-to-noise ratio $(\Delta/\sigma)^2$ is exponential, which has two implications. When the SNR is poor ($\Delta/\sigma \ll 1$, as with perplexity), $\hat{n}$ collapses to approximately 2, giving $\hat{k} = 1$: no beam width helps. When the SNR is moderate ($\Delta/\sigma \gtrsim 2.33$, as may be the case with a trained PRM), $\hat{n}$ reaches 16 or higher, giving $\hat{k} \geq 4$, consistent with our empirical observation that PRM-guided search helps through beam width 4. The exponential sensitivity means that small improvements in scorer quality can dramatically expand the useful search range.

\section{Experiments}
\label{sec:experiments}

\subsection{Setup}

We evaluate on MR-BEN \citep{zeng2024mrben}: 5,975 questions across 10 subjects spanning math, science, medicine, and logic. We use three policy models: \texttt{Qwen2.5-7B-Instruct} \citep{qwen2.5_hf}, \texttt{Llama-3.1-8B-Instruct} \citep{llama3.1_hf}, and \texttt{Mistral-7B-Instruct-v0.3} \citep{mistral_hf}. For scoring, we compare: (1)~\textbf{Perplexity}: each model self-scores using negative log-perplexity of the full reasoning trace, so that higher scores indicate greater confidence; (2)~\textbf{MathShepherd PRM}: \texttt{math-shepherd-mistral-7b-prm} \citep{wang2024mathshepherd}, a process reward model outputting step-correctness probabilities. At each depth, beam selection uses only the most recent step's score, not a cumulative aggregate. For perplexity, this is the negative log-perplexity of the entire sequence up to that step. For PRM, this is the correctness probability of the latest step.

We run beam search with beam width $k \in \{1, 2, 3, 4\}$. At each depth, each of the $k$ beams generates $k$ candidate continuations, yielding $n = k^2$ total candidates from which the top $k$ are selected. At $k=1$, the policy produces a single sample per step and no selection is performed; we refer to this as \emph{single-sample decoding}. The maximum depth is 24 and the temperature is 0.7 for all configurations. We run 3 seeds per configuration and report means and standard errors. At $k$=4, a single subject--seed combination requires approximately one week of GPU time on a single A100, making beam widths beyond 4 computationally prohibitive with our current setup.

\subsection{Main Results}

Tables~\ref{tab:perplexity_results} and \ref{tab:prm_comparison} present the central comparison. Perplexity-scored beam search provides no benefit: Qwen is flat ($\pm 0.1$ points), while Llama and Mistral \emph{degrade} at wider beams ($-$5.4 and $-$3.9 points at $k$=4), consistent with the overestimation bias predicted by Lemma~\ref{lem:overestimation}. The model ranking Qwen $>$ Llama $>$ Mistral is stable across all beam widths, confirming that policy capability dominates when the reward signal is unreliable. Standard errors also grow with beam width under perplexity scoring (e.g., Llama: 1.3 at $k$=1 vs.\ 3.4 at $k$=4), reflecting the increasing role of noise-driven selection at wider beams.

\begin{table*}[h!]
\centering

\begin{tabular}{@{}lccccc@{}}
\toprule
\textbf{Model} & $k=1$ & $k=2$ & $k=3$ & $k=4$ & $\Delta$ ($k$=4$-$$k$=1) \\
\midrule
Qwen    & 71.3 $\pm$ 1.0 & 71.1 $\pm$ 1.0 & 71.2 $\pm$ 1.1 & \textbf{71.4} $\pm$ 1.5 & +0.1 \\
Llama   & \textbf{59.7} $\pm$ 1.3 & 58.4 $\pm$ 1.5 & 58.5 $\pm$ 2.1 & 54.3 $\pm$ 3.4 & $-$5.4 \\
Mistral & \textbf{42.4} $\pm$ 1.6 & 42.0 $\pm$ 1.7 & 41.0 $\pm$ 2.1 & 38.5 $\pm$ 2.7 & $-$3.9 \\
\bottomrule
\end{tabular}
\caption{Perplexity-scored beam search: Macro success rate (\%) by model and beam width. Values are mean $\pm$ standard error across subjects and seeds.}
\label{tab:perplexity_results}
\end{table*}

Replacing perplexity with PRM scoring, while keeping the model and algorithm identical, transforms beam search scaling from flat or negative to strongly positive (Table~\ref{tab:prm_comparison}, Figure~\ref{fig:scorer_comparison}). All three models benefit at $k$=4: Mistral gains +8.9 points (42.3\% to 51.2\%), Llama gains +6.2 points (59.2\% to 65.4\%), and Qwen gains +3.4 points (71.7\% to 75.1\%). The largest relative gain occurs for the weakest model (Mistral), suggesting PRM-guided search is especially valuable when the policy alone struggles.

\begin{table*}[h!]
\centering

\begin{tabular}{@{}lccccl@{}}
\toprule
\textbf{Model} & $k=1$ & $k=2$ & $k=3$ & $k=4$ & $\Delta$ (best$-$$k$=1) \\
\midrule
Qwen    & 71.7 $\pm$ 1.2 & 73.8 $\pm$ 1.3 & 73.9 $\pm$ 1.4 & \textbf{75.1} $\pm$ 1.3 & +3.4 \\
Llama   & 59.2 $\pm$ 1.7 & 63.4 $\pm$ 1.8 & 64.7 $\pm$ 1.8 & \textbf{65.4} $\pm$ 2.4 & +6.2 \\
Mistral & 42.3 $\pm$ 2.7 & 50.8 $\pm$ 3.3 & 47.9 $\pm$ 3.6 & \textbf{51.2} $\pm$ 3.7 & +8.9 \\
\bottomrule
\end{tabular}
\caption{PRM-scored beam search: Macro success rate (\%) by model and beam width. Values are mean $\pm$ standard error across subjects and seeds.}
\label{tab:prm_comparison}
\end{table*}

\begin{figure*}[t]
    \centering
    \includegraphics[width=\linewidth]{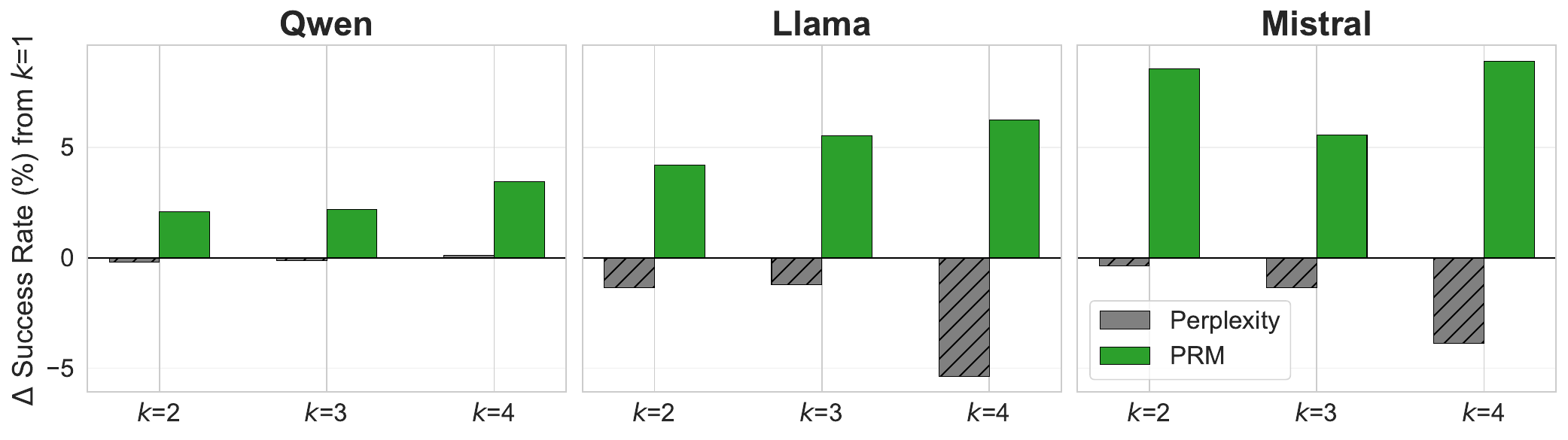}
    \caption{Per-model change from $k$=1 by scorer. Hatched bars indicate negative $\Delta$. The PRM unlocks search benefits where perplexity cannot.}
    \label{fig:delta_comparison}
\end{figure*}

\subsection{Interpreting Results Through the Maximum Useful Beam Width}
Corollary~\ref{cor:kstar} predicts a maximum useful beam width $\hat{k}$ for each scorer. Our results are consistent with:
\begin{itemize}[leftmargin=*,topsep=2pt,itemsep=2pt]
    \item \textbf{Perplexity}: $\hat{k} = 1$ for all three models. The predicted candidate pool size $\hat{n} \approx 2$ means even a single incorrect competitor can mislead the scorer, and no tested beam width produces improvement. Llama and Mistral actively degrade at wider beams.
    \item \textbf{PRM}: $\hat{k} \geq 4$ for all three models. All peak at $k$=4. Qwen and Llama scale monotonically; Mistral dips at $k$=3 before recovering at $k$=4, likely reflecting higher variance at wider beams rather than a true optimum near $k$=2.
\end{itemize}
The perplexity and PRM results are not two separate findings; they are two points on the $\hat{k}$ continuum predicted by a single formula. Figure~\ref{fig:delta_comparison} visualizes the contrast: PRM bars rise while perplexity bars are flat or negative. Per-subject breakdowns appear in Appendix~\ref{app:per_subject}.

\section{Case Study: Diagnosing Reward Inversion}
\label{sec:case_study}

The aggregate results above show \emph{that} scorer quality determines beam search effectiveness. This section examines \emph{how} the failure unfolds at the level of individual beam selections. We trace beam search behavior for \texttt{Mistral-7B-Instruct-v0.3} on \texttt{high\_school\_biology} with perplexity scoring, where $k$=4 scored 59.7\% versus 63.0\% for single-sample decoding, despite 4$\times$ the search effort (see also Figures~\ref{fig:case_study_acc} and \ref{fig:case_study_inversion} in Appendix~\ref{app:case_study_figs}).

\begin{figure}[t]
    \centering
    \includegraphics[width=\columnwidth]{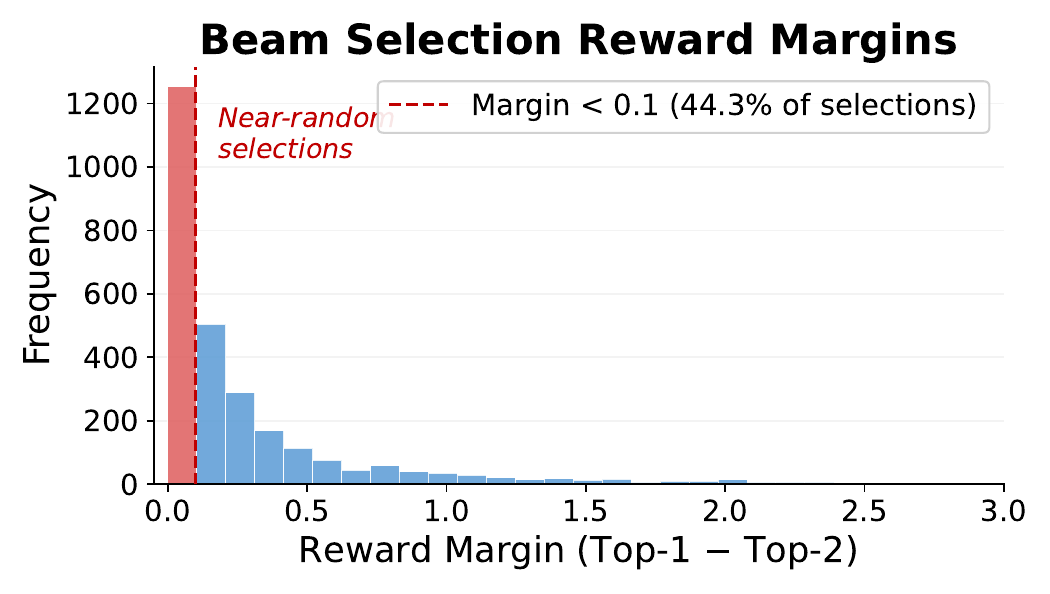}
    \caption{Reward margin distribution for perplexity-scored beam selections. The peak near zero confirms that 44\% of selections have negligible margins, signaling that the scorer cannot support wider search.}
    \label{fig:reward_margin_dist}
\end{figure}

\subsection{Reward Inversion at the Trace Level}

Problem \texttt{cf68a215} illustrates the mechanism concretely. At depth 2, the beam selection log shows:
{\small
\begin{verbatim}
[BEAM_SELECT] depth=2 | candidates=12
  top_rewards=[-3.04, -3.05, -3.18, ...]
  #1: reward=-3.0405 [ANS:C] <- SELECTED
  #2: reward=-3.0496 [ANS:A] <- REJECTED
\end{verbatim}
}
The correct path (A) was rejected in favor of the incorrect path (C) by a margin of just 0.009. Among 12 candidates, the maximization inflated the score of candidate C by enough to overtake A---a microscopic margin that a more discriminative scorer would not produce.

This is not an isolated case. In 25 disagreement cases where $k$=1 was correct but $k$=4 was wrong, the $k$=4 path consistently had a higher reward, i.e.\ lower perplexity, as shown in Table~\ref{tab:reward_inversion}. The pattern is systematic: beam search selects paths that look better to the scorer but are worse in quality, precisely the reward inversion predicted by Lemma~\ref{lem:overestimation}.

\begin{table}[h!]

\begin{tabular}{@{}lcccc@{}}
\toprule
\textbf{Problem} & \textbf{$k=1$} & \textbf{Correct} & \textbf{$k=4$} & \textbf{Correct} \\
\midrule
cf68a215 & $-6.71$ & \cmark & $-2.84$ & \xmark \\
c1eacc33 & $-4.65$ & \cmark & $-2.19$ & \xmark \\
5c9182fd & $-3.70$ & \cmark & $-2.53$ & \xmark \\
681201fd & $-4.49$ & \cmark & $-2.82$ & \xmark \\
\bottomrule
\end{tabular}
\caption{Reward inversion: $k$=4 selects higher-scoring but incorrect paths.}
\label{tab:reward_inversion}
\end{table}

\subsection{Score Margins as a Diagnostic}

The trace-level analysis above reveals why perplexity scoring fails. Score margins reveal how such failures can be detected without ground-truth labels. Across 2,770 beam selections in this domain, 44.3\% had a reward margin $< 0.1$ between the top two candidates (Figure~\ref{fig:reward_margin_dist}), indicating near-random selection. When margins are this small, the scorer cannot reliably separate correct from incorrect paths, and beam selection reduces to a noisy coin flip.

This observation connects directly to the practical guidance in Section~\ref{sec:discussion}: inspecting score margins on a small pilot run provides a cheap diagnostic for whether a scorer can support wider search. Consistently small margins signal that the effective quality gap $\Delta_{\mathrm{eff}}$ is near zero and $\hat{k}$ is likely low.

\section{Discussion}
\label{sec:discussion}

\subsection{The Optimal Beam Width as a Design Parameter}
Our central result, the predicted maximum useful beam width $\hat{k}$ (Corollary~\ref{cor:kstar}), reframes beam search from a ``wider is better'' heuristic to a principled design decision. The key insight is that $\hat{k}$ depends exponentially on the squared SNR of the scorer: small improvements in scorer quality can dramatically expand the useful beam range, while even state-of-the-art scorers have a finite $\hat{k}$ beyond which additional search is wasted. The single-sample case $k$=1 follows naturally: with no maximization there is no overestimation bias, and the policy's implicit calibration is preserved.

\subsection{Practical Guidance}
The formula $\hat{n} = 1 + \exp(\Delta^2 / 2\sigma^2)$ characterizes $\hat{k}$ in terms of two quantities, the quality gap $\Delta$ and scorer noise $\sigma$, that are not directly observable. We do not claim that $\hat{k}$ can be computed exactly from data. Rather, the analysis provides three forms of guidance:
\begin{enumerate}[leftmargin=*,topsep=2pt,itemsep=2pt]
    \item \textbf{Diagnostic indicators.} Score margins at beam selection provide a qualitative signal. If margins are consistently small, as we observe for 44\% of perplexity selections, the scorer cannot reliably distinguish candidates and $\hat{k}$ is likely low. The dispersion of scores reflects both genuine quality variation and scorer noise, so it serves as a rough proxy for scorer unreliability rather than an estimate of $\sigma$ alone.
    \item \textbf{Pilot runs.} A practical approach is to compare $k$=2 against $k$=1 on a small validation set. If no improvement is observed, the scorer's SNR is too low to support wider search. On problems with ground-truth labels, one can estimate $\Delta/\sigma$ from the separation between score distributions of correct and incorrect candidates.
    \item \textbf{Invest in the scorer, not the beam width.} Since $\hat{n}$ grows exponentially with $(\Delta/\sigma)^2$, reducing $\sigma$ is far more effective than increasing $k$. Investing in reward model quality expands the useful search range, while widening the beam beyond $\hat{k}$ actively degrades performance.
\end{enumerate}

\subsection{Toward Bias-Corrected Beam Search}
Our theory also suggests a correction that could raise $\hat{k}$: rather than selecting $\argmax_i R_i$, subtract the estimated overestimation bias from the maximum. Concretely, one could penalize the best-of-$n$ score by $B(\hat{\sigma}, n-1)$, where $\hat{\sigma}$ is estimated online from the variance of scores across candidates at each selection step. This would effectively increase $\Delta_{\mathrm{eff}}$, raising $\hat{n}$ and thus $\hat{k}$ without changing the scorer. We leave empirical validation to future work.

\section{Conclusion}

This study provides a theoretical and empirical answer to the question: \emph{how wide should you search?} Our EVT-based analysis shows that beam selection introduces a systematic overestimation bias proportional to scorer noise, yielding a predicted maximum useful candidate pool $\hat{n} = 1 + \exp(\Delta^2 / 2\sigma^2)$ and corresponding beam width $\hat{k} = \lfloor\sqrt{\hat{n}}\rfloor$ that depend on the scorer's signal-to-noise ratio. Beyond $\hat{k}$, wider search degrades performance. This analysis correctly predicts our empirical findings: perplexity scoring, with $\hat{k} = 1$, provides zero benefit at any beam width across three models and ten domains, while PRM scoring, with $\hat{k} \geq 4$, unlocks meaningful gains of up to 8.9 percentage points.

The practical message is that beam width is not a hyperparameter to maximize; it has a theoretically grounded optimum determined by scorer quality. While $\Delta$ and $\sigma$ are not directly observable, diagnostic indicators such as score margins and pilot-run comparisons can help judge whether a given scorer supports wider search.

\paragraph{Reproducibility.} Our code is publicly available at \texttt{[anonymous]}, built as a fork of the LLM Reasoners library \citep{hao2024llm}.

\section*{Limitations}

We note limitations of both the theory and the experimental scope.

\paragraph{Two-class quality model.} Assumption~\ref{ass:two_class} collapses candidate quality to a binary correct/incorrect partition with a single gap $\Delta$. Real candidate pools exhibit a spectrum of quality levels. As noted in Remark~\ref{rem:stylized}, this makes our bound conservative: multiple correct candidates or a quality continuum generally makes beam selection easier, not harder.

\paragraph{Scale and beam width range.} We evaluate only 7B-parameter models and beam widths up to $k = 4$. The theory predicts that $\hat{k}$ can be much larger for high-SNR scorers. Testing $k > 4$ and larger models would be needed to empirically locate the turning point for PRM-guided search. We also use a single temperature of $T = 0.7$ throughout. Temperature affects candidate diversity and the effective independence of the candidate pool, which may modulate the relationship between $n$ and overestimation bias.

\paragraph{Intrinsic signal coverage.} Among training-free scoring signals, we evaluate only perplexity. Other intrinsic signals, from entropy-based self-certainty \citep{kang2024selfcertainty} to Markovian reward adjustments \citep{ma2025metacognition}, may offer better signal-to-noise ratios. Extending the empirical comparison to these alternatives is a natural next step.

\section*{Ethical Considerations}
This work analyzes inference-time search strategies for LLM reasoning. We do not collect or use personal data, and all experiments use publicly available models and benchmarks. The primary ethical consideration is computational cost: beam search at width $k$ requires $O(k^2)$ scorer evaluations per step. Our work mitigates this concern by identifying when wider search is wasteful, helping avoid unnecessary computation.

\newpage
\bibliography{references}

\newpage
\onecolumn
\appendix

\section{Proof of Lemma~\ref{lem:overestimation}}
\label{app:proof_overestimation}

\begin{proof}
    By Assumption~\ref{ass:two_class}, the $n-1$ incorrect candidates have scores $R_j = \mu_w + \epsilon_j$ with $\epsilon_j \overset{\text{i.i.d.}}{\sim} \mathcal{N}(0, \sigma^2)$. By the Fisher--Tippett--Gnedenko theorem \citep{fisher1928limiting, gnedenko1943distribution, coles2001introduction}, the maximum of $n-1$ i.i.d.\ Gaussian draws converges to a Gumbel distribution. The approximation is tight even for small $n$, with relative error below 2\% for $n \geq 4$. The location and scale parameters are $\mu^{\mathrm{GEV}} = \mu_w + \sigma \Phi^{-1}(1 - 1/(n-1))$ and $\sigma^{\mathrm{GEV}} = \sigma[\Phi^{-1}(1 - 1/(\mathrm{e}(n-1))) - \Phi^{-1}(1 - 1/(n-1))]$. The expectation of the Gumbel distribution is $\mu^{\mathrm{GEV}} + \gamma_{\mathrm{EM}} \sigma^{\mathrm{GEV}}$, giving \eqref{eq:overestimation}--\eqref{eq:bias_term}. The approximation \eqref{eq:bias_approx} follows from $\Phi^{-1}(1 - 1/(n-1)) \approx \sqrt{2\log(n-1)} - 1/2$ for large $n$ \citep{blair1976impossibility}, with the Euler--Mascheroni correction becoming negligible. The correct candidate is a single draw, so $\mathbb{E}[R_c] = \mu_c$ with no max-induced bias.
\end{proof}

\section{Proof of Theorem~\ref{thm:suboptimal}}
\label{app:proof_suboptimal}

\begin{proof}
We bound $\Pr(\max_{j=1,\dots,n-1} R_j > R_c)$ where $R_c = \mu_c + \epsilon_c$ with $\epsilon_c \sim \mathcal{N}(0, \sigma^2)$ and $\max_j R_j$ follows a GEV distribution as established in Lemma~\ref{lem:overestimation}.

Define $Z := \max_{j} R_j - R_c$. We need to bound $\Pr(Z > 0)$.

By Lemma~\ref{lem:overestimation}, $\max_j R_j \sim \text{GEV}(\mu^{\text{GEV}}, \sigma^{\text{GEV}}, 0)$ with parameters as specified. Since $\max_j R_j$ and $R_c$ are independent, the difference $Z$ has:
\begin{align*}
    \mathbb{E}[Z] &= (\mu_w + B(\sigma, n-1)) - \mu_c = -\Delta + B(\sigma, n-1) = -\Delta_{\mathrm{eff}}, \\
    \text{Var}[Z] &= \underbrace{\frac{\pi^2}{6}(\sigma^{\text{GEV}})^2}_{\text{variance of } \max_j R_j} + \underbrace{\sigma^2}_{\text{variance of } R_c}.
\end{align*}

We bound $\text{Var}[Z]$ from above. The variance of the maximum of $m$ i.i.d.\ $\mathcal{N}(\mu, \sigma^2)$ random variables is at most $\sigma^2$ for all $m \geq 1$ (it equals $\sigma^2$ when $m = 1$ and decreases monotonically with $m$; see \citealt{david2003order}, \S5.4). Therefore $\text{Var}[Z] \leq \sigma^2 + \sigma^2 = 2\sigma^2$.

Applying the Cantelli inequality (one-sided Chebyshev):
\begin{equation*}
    \Pr(Z > 0) = \Pr(Z - \mathbb{E}[Z] > \Delta_{\mathrm{eff}}) \leq \frac{\text{Var}[Z]}{\text{Var}[Z] + \Delta_{\mathrm{eff}}^2} \leq \frac{2\sigma^2}{2\sigma^2 + \Delta_{\mathrm{eff}}^2} = \left(1 + \frac{\Delta_{\mathrm{eff}}^2}{2\sigma^2}\right)^{-1}.
\end{equation*}

The bound is controlled by the squared ratio $(\Delta_{\mathrm{eff}}/\sigma)^2$. As $n$ grows, the overestimation bias $B(\sigma, n-1)$ increases, shrinking $\Delta_{\mathrm{eff}}$ and driving the bound toward~1. At $n = \hat{n}$, we have $\Delta_{\mathrm{eff}} = 0$ and the bound becomes vacuous, consistent with the criterion in Corollary~\ref{cor:criterion}.
\end{proof}

\section{Per-Subject Curves by Model}
\label{app:per_subject}

\subsection{Perplexity Scoring}

\begin{figure}[H]
    \centering
    \includegraphics[width=\linewidth]{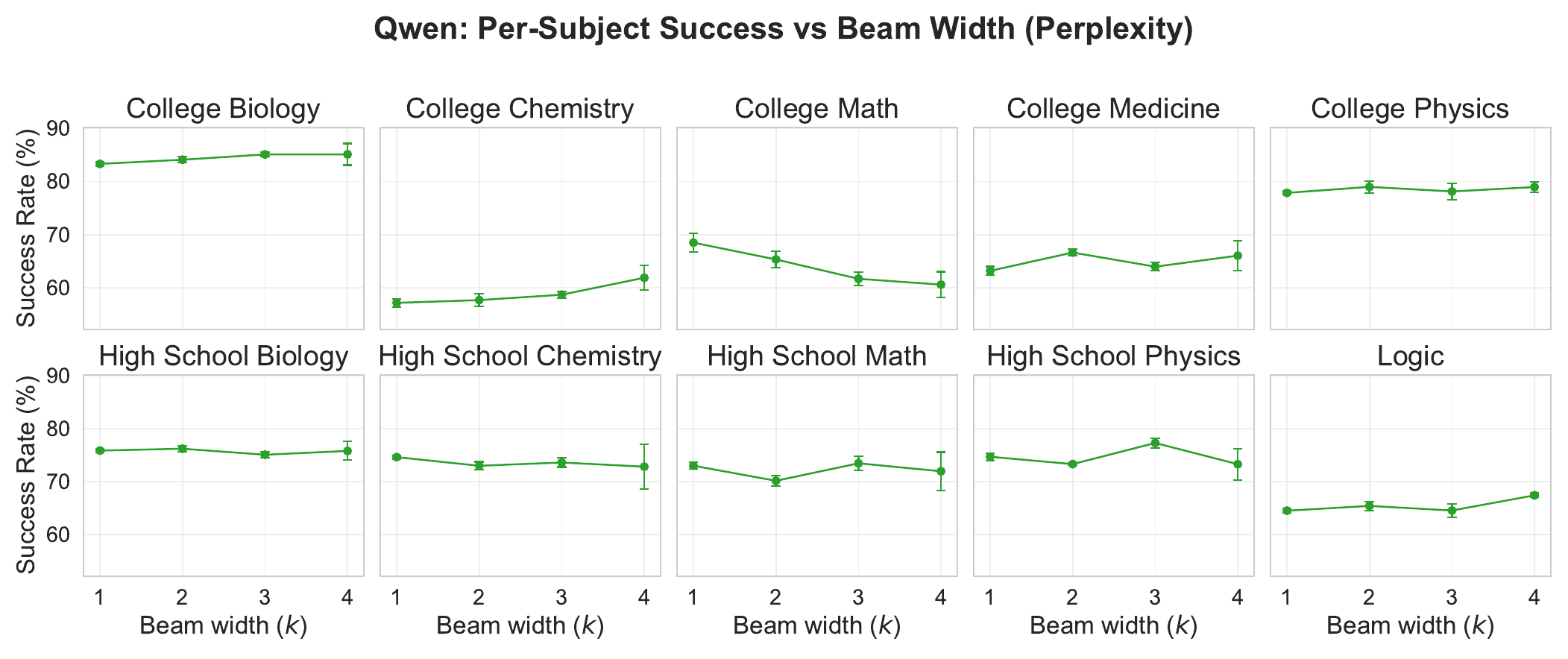}
    \caption{Per-subject success rate vs.\ beam width for \textbf{Qwen2.5-7B-Instruct} with perplexity scoring.}
    \label{fig:qwen_subjects}
\end{figure}

\begin{figure}[H]
    \centering
    \includegraphics[width=\linewidth]{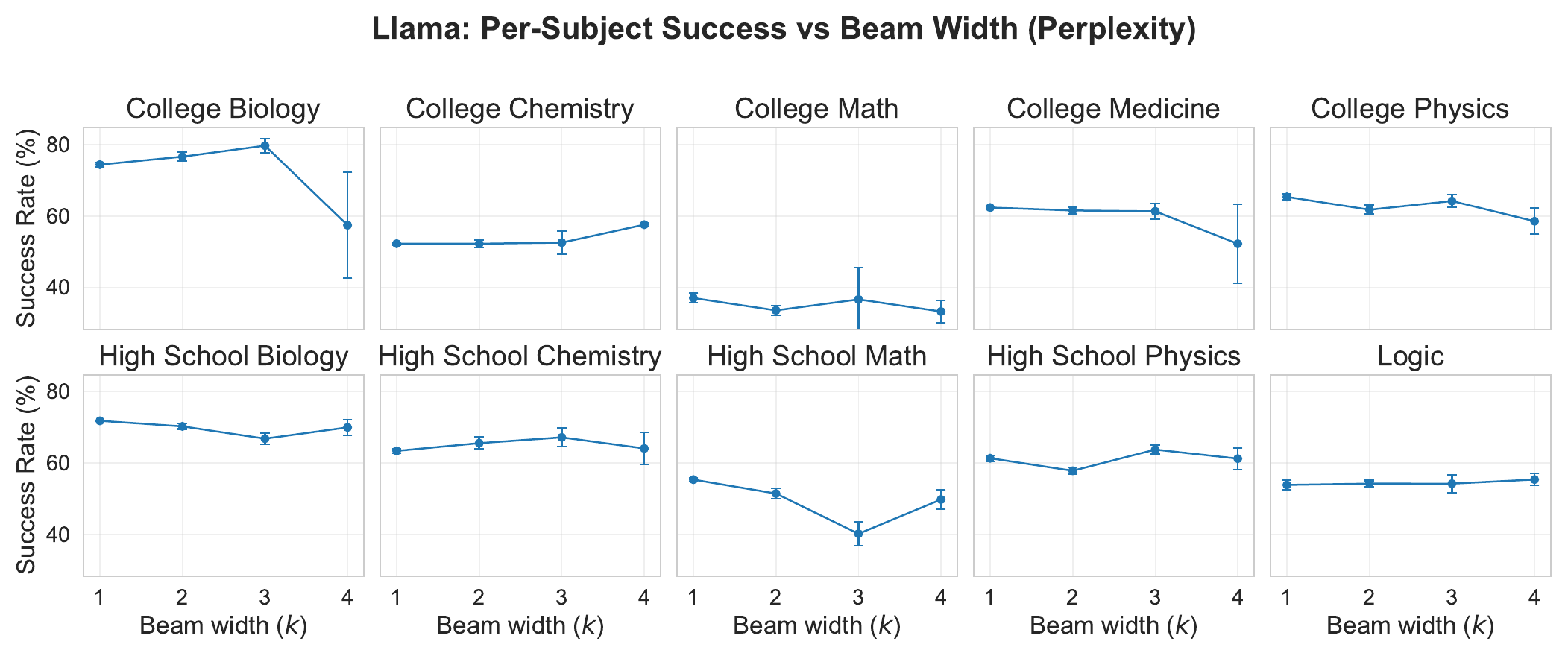}
    \caption{Per-subject success rate vs.\ beam width for \textbf{Llama-3.1-8B-Instruct} with perplexity scoring.}
    \label{fig:llama_subjects}
\end{figure}

\begin{figure}[H]
    \centering
    \includegraphics[width=\linewidth]{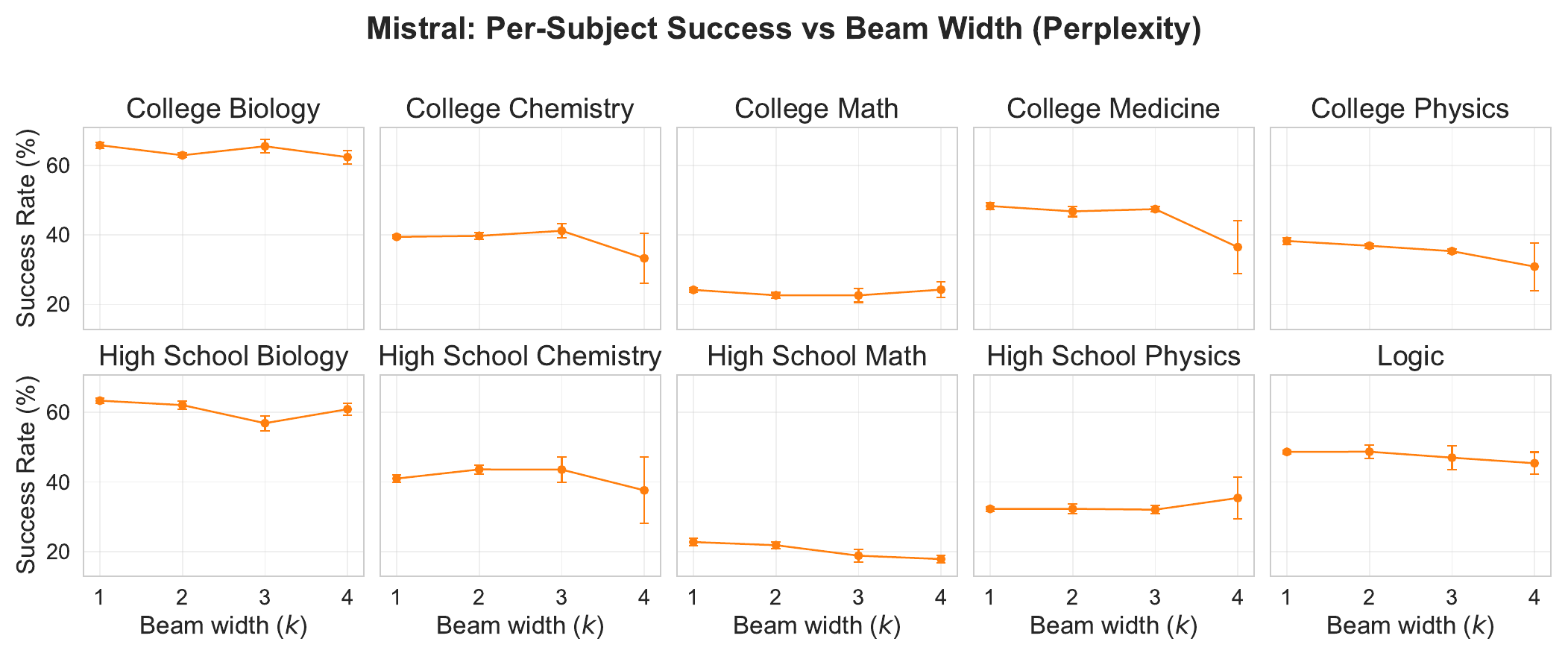}
    \caption{Per-subject success rate vs.\ beam width for \textbf{Mistral-7B-Instruct-v0.3} with perplexity scoring.}
    \label{fig:mistral_subjects}
\end{figure}

\subsection{PRM Scoring}

\begin{figure}[H]
    \centering
    \includegraphics[width=\linewidth]{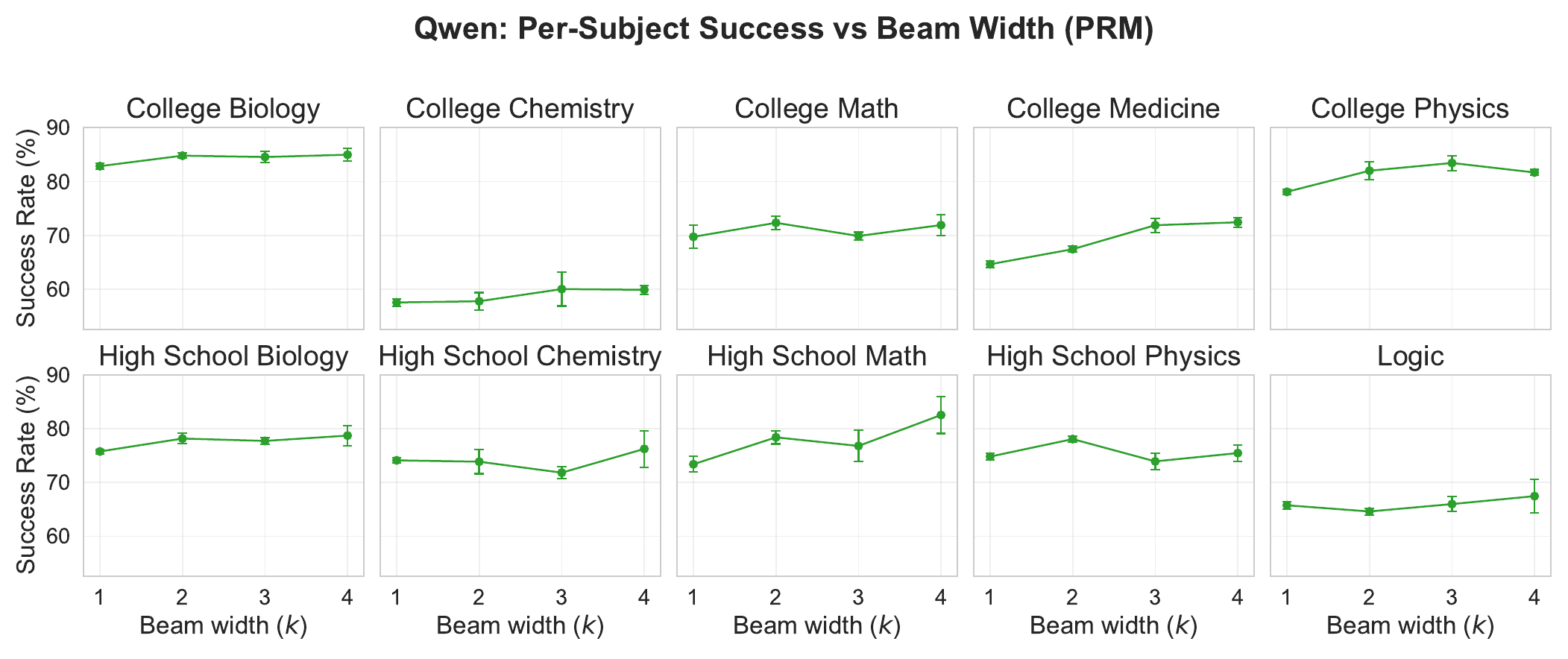}
    \caption{Per-subject success rate vs.\ beam width for \textbf{Qwen2.5-7B-Instruct} with PRM scoring.}
    \label{fig:qwen_subjects_prm}
\end{figure}

\begin{figure}[H]
    \centering
    \includegraphics[width=\linewidth]{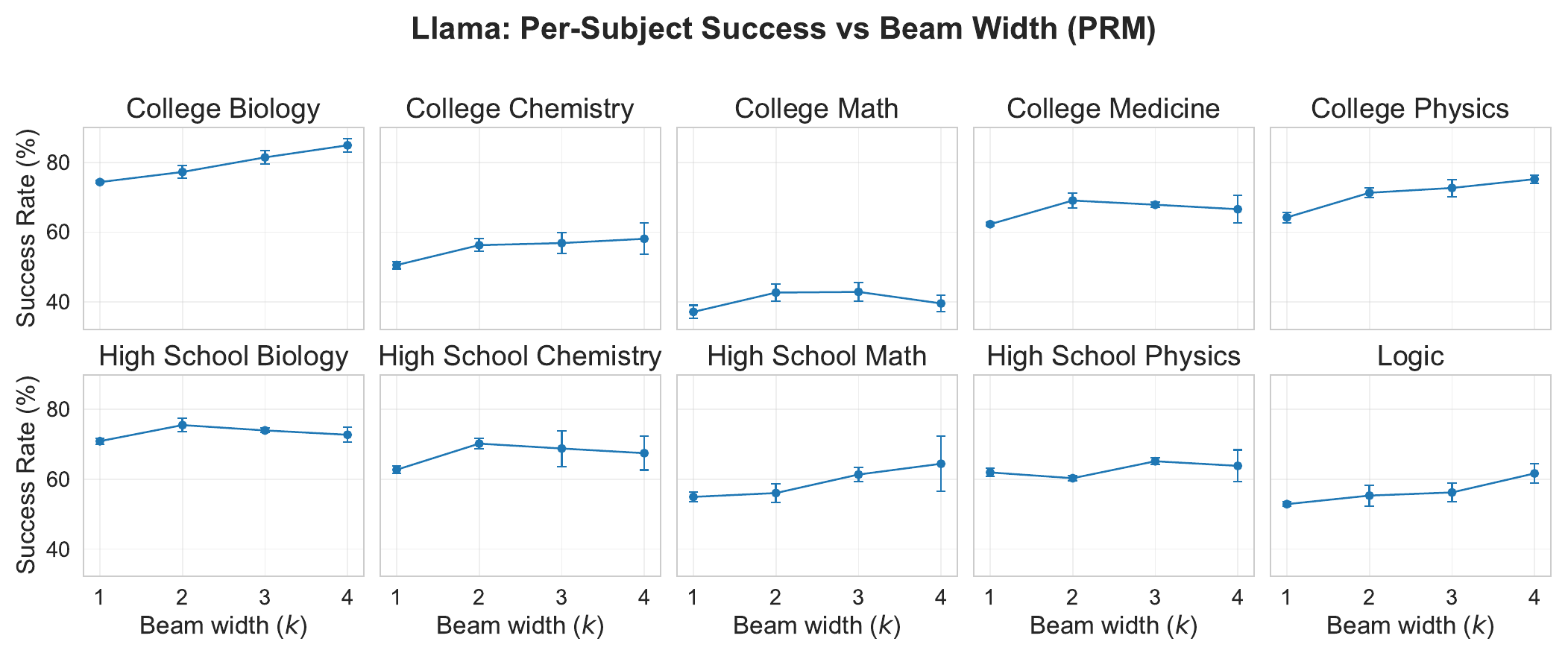}
    \caption{Per-subject success rate vs.\ beam width for \textbf{Llama-3.1-8B-Instruct} with PRM scoring.}
    \label{fig:llama_subjects_prm}
\end{figure}

\begin{figure}[H]
    \centering
    \includegraphics[width=\linewidth]{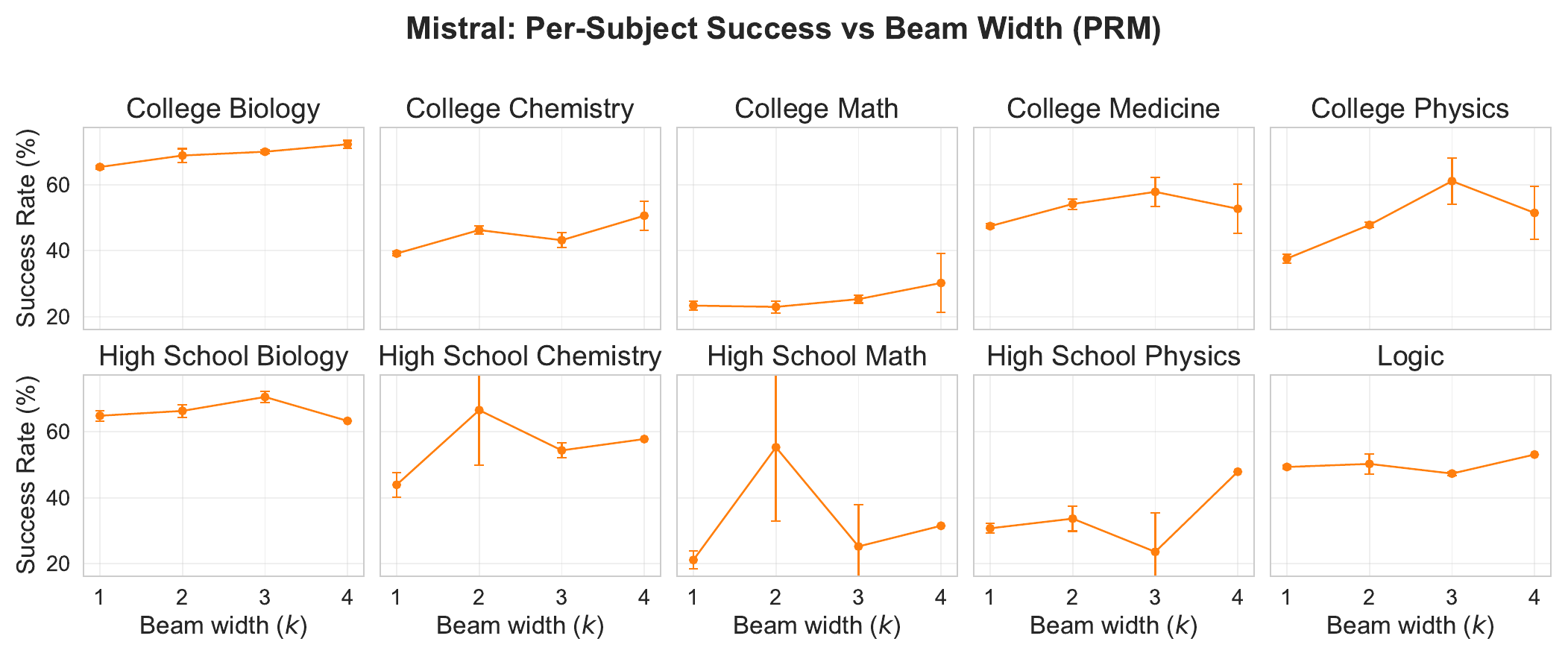}
    \caption{Per-subject success rate vs.\ beam width for \textbf{Mistral-7B-Instruct-v0.3} with PRM scoring.}
    \label{fig:mistral_subjects_prm}
\end{figure}

\section{Implementation Details}
\begin{itemize}
    \item \textbf{Beam Width}: $k \in \{1, 2, 3, 4\}$.
    \item \textbf{Number of Actions}: Equal to beam width.
    \item \textbf{Maximum Depth}: 24 steps.
    \item \textbf{Temperature}: 0.7 for all configurations (including $k$=1).
    \item \textbf{Reward Aggregation}: \texttt{last}: only the most recent step's score is used for beam selection.
    \item \textbf{Perplexity}: Computed over the full sequence (prompt + all steps) at each depth.
    \item \textbf{PRM}: \texttt{math-shepherd-mistral-7b-prm}; uses the correctness probability at the last step tag.
\end{itemize}
All models are loaded from the Hugging Face Hub. Experiments are managed via Weights \& Biases.

\section{Additional Case Study Figures}
\label{app:case_study_figs}

\begin{figure}[H]
    \centering
    \includegraphics[width=0.6\linewidth]{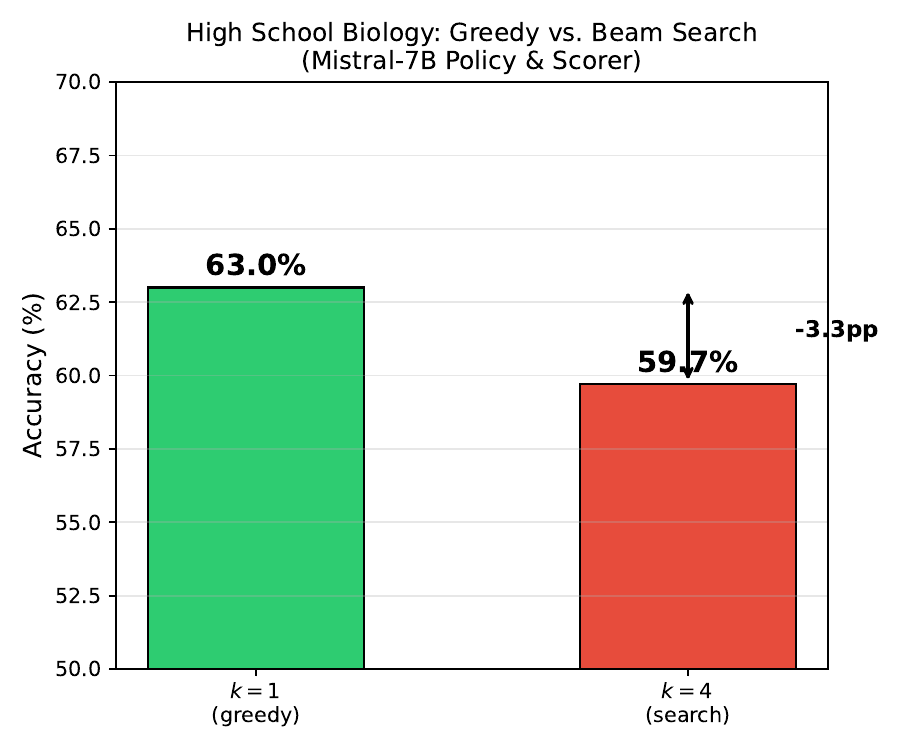}
    \caption{Accuracy: single-sample ($k$=1) vs.\ beam ($k$=4) on \texttt{high\_school\_biology}. Beam search underperforms despite evaluating $k^2 = 16$ candidates per step.}
    \label{fig:case_study_acc}
\end{figure}

\begin{figure}[H]
    \centering
    \includegraphics[width=0.8\linewidth]{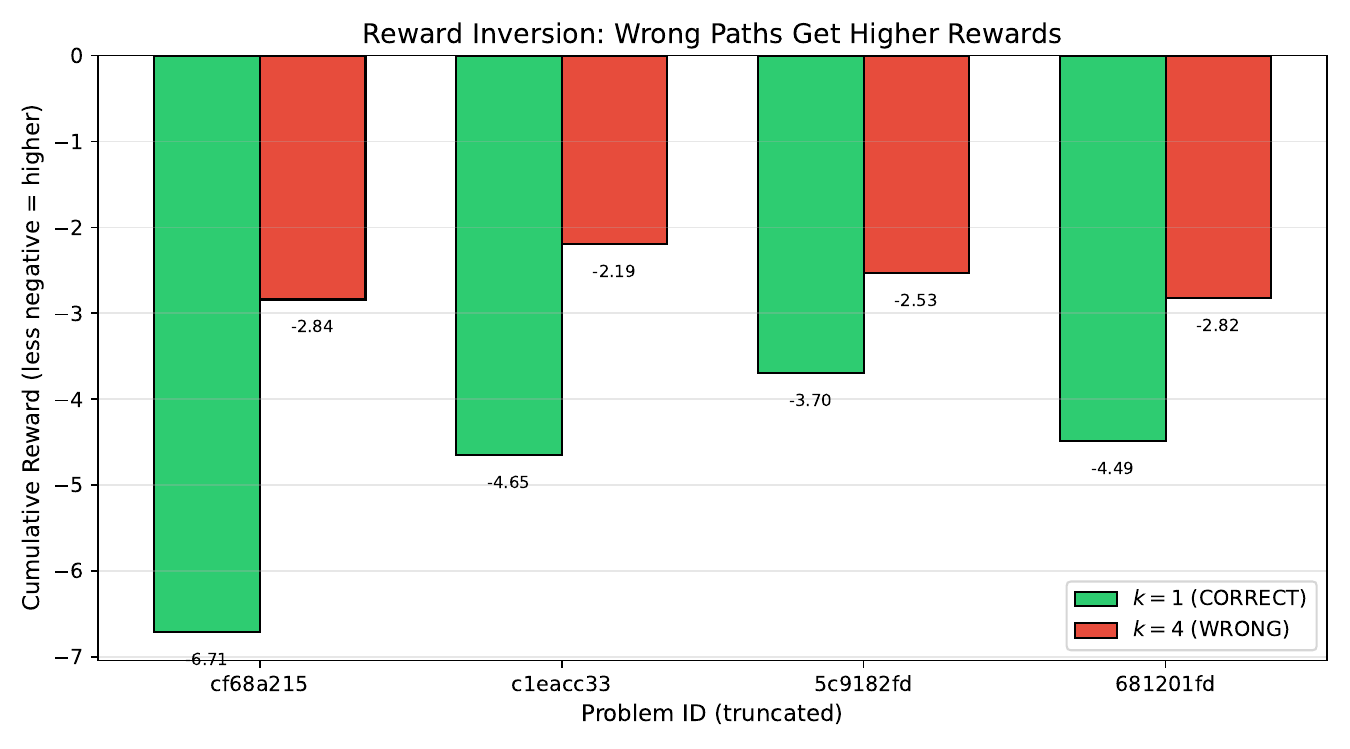}
    \caption{Reward inversion: incorrect paths selected by $k$=4 have higher rewards than correct paths from $k$=1.}
    \label{fig:case_study_inversion}
\end{figure}

\end{document}